\documentclass{article}

\usepackage[final,main]{neurips_2026}

\usepackage[utf8]{inputenc}
\usepackage[T1]{fontenc}
\usepackage{microtype}
\usepackage{graphicx}
\usepackage{subcaption}
\usepackage{booktabs} 
\usepackage{makecell}
\usepackage{url}
\usepackage{amsfonts}
\usepackage{nicefrac}
\usepackage{xcolor}
\usepackage{wrapfig}
\usepackage{needspace}

\usepackage{hyperref}
\usepackage{amsmath}
\usepackage{amssymb}
\usepackage{mathtools}
\usepackage{amsthm}
\usepackage{placeins}
\usepackage{enumitem}
\usepackage{multirow}
\setlist[enumerate, 1]{label=(A\arabic*)}  
\usepackage{colortbl}

\definecolor{smallscale}{RGB}{238,246,255}
\definecolor{largescale}{RGB}{240,248,241}
\definecolor{smallscalehead}{RGB}{218,234,255}
\definecolor{largescalehead}{RGB}{221,240,224}

\definecolor{credithead}{RGB}{225,235,250}
\definecolor{creditrow}{RGB}{244,248,253}
\definecolor{routinghead}{RGB}{224,242,231}
\definecolor{routingrow}{RGB}{244,250,246}
\definecolor{emahead}{RGB}{247,235,221}
\definecolor{emarow}{RGB}{253,248,242}

\usepackage[capitalize,noabbrev]{cleveref}
\theoremstyle{plain}
\newtheorem{theorem}{Theorem}[section]

\newtheorem{lemma}[theorem]{Lemma}

\theoremstyle{definition}

\theoremstyle{remark}

\DeclareMathOperator{\softmax}{softmax}
\DeclareMathOperator{\sg}{sg}

\usepackage{algorithmic}
\usepackage{algorithm}

\usepackage[textsize=tiny]{todonotes}

\title{Distributionally Robust Mixture-of-Experts Training}

\author{%
  Xin Teng \\
  New York University \\
  Center for Data Science, NYU Shanghai\\
  \texttt{xt2251@nyu.edu} \\
  \And
  Muxiao Li \\
  New York University \\
  Center for Data Science, NYU Shanghai\\
  \texttt{ml9007@nyu.edu} \\
  \And
  Hongyi Wen \\
  New York University \\
  Center for Data Science, NYU Shanghai\\
  \texttt{hongyi.wen@nyu.edu} \\
}

\begin{document}

\maketitle

\begin{abstract}
Mixture-of-Experts (MoE) transformers scale capacity by activating only a few experts per token, but this sparsity creates a hidden reliability problem: when routing is imperfect, load-balanced models may send tokens to experts that are insufficiently trained for the assigned inputs. We propose \emph{Distributionally Robust MoE Training} (DRMoET), a drop-in objective that treats layer-wise experts as endogenous robustness groups and optimizes high-loss routing outcomes rather than merely equalizing traffic. DRMoET updates a per-layer expert distribution by an entropy-regularized softmax rule on EMA-smoothed, activation-weighted expert losses, strengthening plausible non-top routing paths while preserving standard MoE computation. Under the FLAME-MoE recipe at 746M-total and 10.3B-total scales, DRMoET improves downstream averages over both standard FLAME-MoE and auxiliary-loss-free balancing. At 10.3B total parameters and 67B training tokens, DRMoET improves the seven-task average from 0.6625 to 0.6767, while the auxiliary-loss-free baseline achieves 0.6431. Mechanistic analyses show lower expert-loss variance with nearly unchanged mean loss, 4.3\% lower excess loss under forced mid-$k$ misrouting, and improved domain--expert specialization. These results position routing robustness--not only utilization balance--as a practical objective for reliable sparse MoE scaling. Project page and code are available at: https://drmoet.github.io/.
\end{abstract}

\section{Introduction}
\label{sec:introduction}

Scaling has been a central driver of progress in deep generative language models, but na\"ively increasing parameter count quickly becomes computationally prohibitive. Mixture-of-Experts (MoE) architectures offer an appealing alternative: they increase \emph{total} model capacity while keeping \emph{per-token} computation nearly constant by routing each token through only a small subset of expert subnetworks. This sparse activation principle underlies widely used MoE systems, including Switch Transformer~\citep{fedus2022switch} and GShard~\citep{lepikhin2021}, and continues to be adopted in modern large-scale language models.

Despite their promise, MoE models introduce a hidden reliability problem: sparse routing exposes the model to the quality of whichever experts are selected. The router need not fail catastrophically for this to matter; imperfect gating, load-balancing pressure, or distribution shift can send tokens to plausible but non-optimal experts. Under standard training, these lower-ranked or less frequently favored experts receive weaker task-aligned updates and can remain noticeably worse than top experts. As a result, occasional misrouting causes disproportionate loss. The core issue is therefore not merely uneven utilization, but uneven expert competence under imperfect routing.

\begin{figure*}[t]
    \centering
    \includegraphics[width=0.95\linewidth]{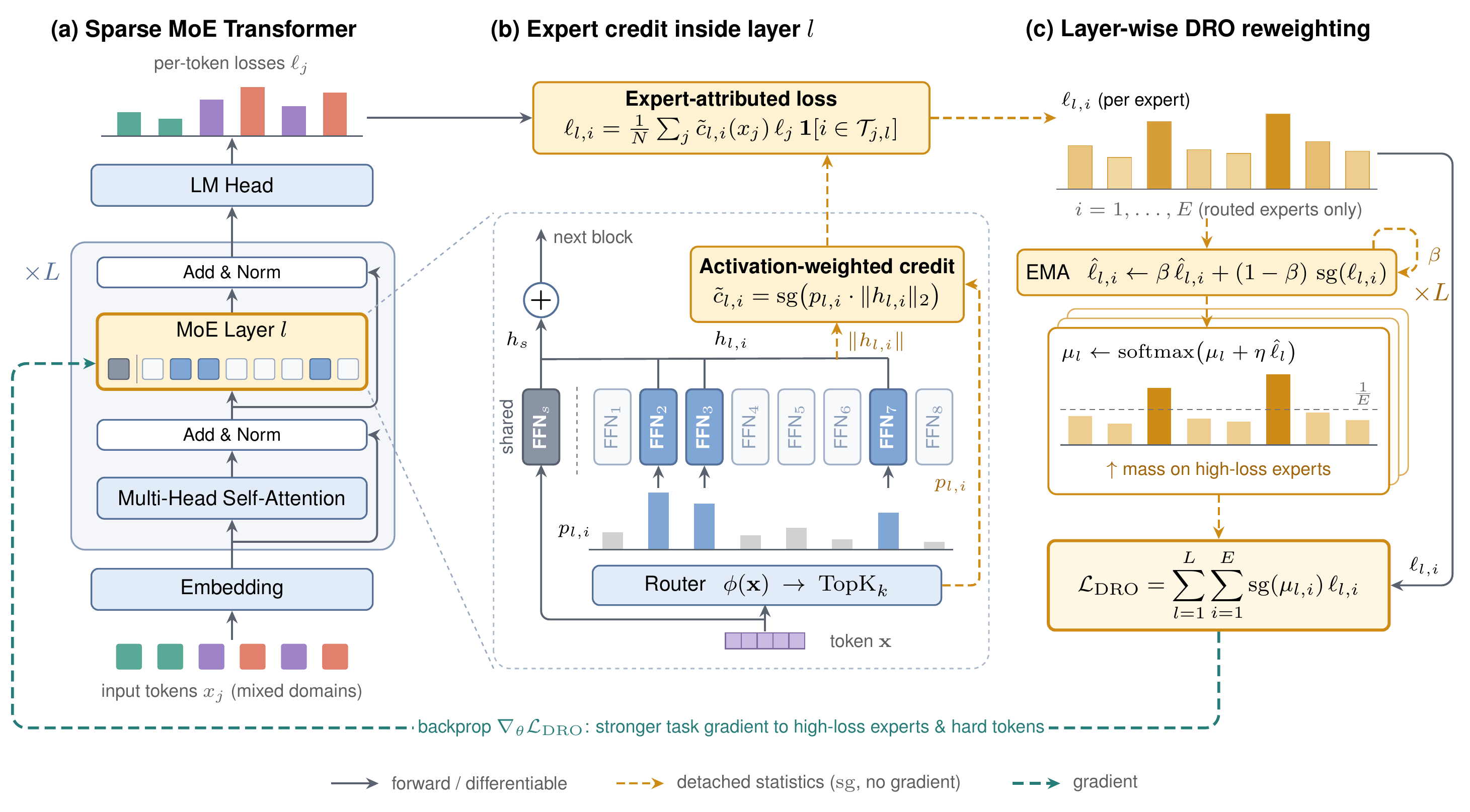}
    \caption{Distributionally Robust MoE Training workflow. DRMoET treats experts as layer-wise robustness groups and adaptively upweights high-loss routing outcomes, strengthening plausible non-optimal paths rather than only encouraging uniform utilization.}
    \label{fig:teaser}
\end{figure*}

A common mitigation is to add auxiliary load-balancing terms that penalize skewed routing and encourage more even expert utilization~\citep{fedus2022switch,lepikhin2021}. While effective at preventing extreme collapse, these regularizers optimize allocation rather than \textit{competence}: balanced traffic does not ensure that non-top experts are capable when selected. Moreover, because expert specialization is a desirable property in MoE models, uniformly pushing routing toward all experts can conflict with selective activation when certain experts are genuinely better suited to particular tokens. This motivates a complementary objective that improves worst routing outcomes without erasing useful router preferences.

Recent work has expanded MoE research beyond routing collapse and load balancing.
On the expert specialization side, several approaches explicitly encourage clearer expert roles and more stable specialization, rather than only enforcing uniform utilization \citep{krishnamurthy2023improving,dai2024deepseekmoe}. In parallel, large-scale MoE LLMs and training recipes emphasize that expert quality and specialization remain central to scaling, not just utilization \citep{jiang2024mixtral}. Beyond FFN-MoE, head or attention-level MoE treats attention heads as experts, pushing sparsity and specialization into the attention mechanism itself \citep{jin2024moh,wu2024multi}. For adaptation, recent work studies parameter-efficient MoE variants that preserve expert structure while substantially reducing trainable parameters \citep{zadouri2023pushing}.

In this work, we propose \emph{Distributionally Robust MoE Training} (DRMoET), a training framework that makes routing robustness explicit. Unlike standard Group DRO~\citep{sagawa2019distributionally}, whose groups are predefined, DRMoET treats each layer's routed experts as endogenous groups induced by the current router. It maintains a per-layer distribution over experts and shifts mass toward high-loss expert--input interactions using EMA-smoothed, activation-weighted losses. The resulting objective sends stronger task gradients to hard-domain tokens and plausible non-top experts while leaving the router, architecture, and sparse forward/backward computation unchanged.

This formulation makes predictions beyond aggregate accuracy: it should reduce expert-quality dispersion, improve behavior when non-top experts are selected, and preserve rather than collapse specialization. We evaluate DRMoET in pretraining under the FLAME-MoE recipe at two model scales: \texttt{FLAME-MoE-290M-746M} and \texttt{FLAME-MoE-1.7B-10.3B}. Our experiments support all three: DRMoET reduces worst-expert loss and expert-loss variance while keeping mean expert loss nearly unchanged; a forced mid-$k$ misrouting probe shows 4.3\% less degradation; and functional specialization metrics, including expert--domain mutual information and competence advantage, improve under DRMoET. Across downstream benchmarks, mechanism probes, and ablations, the evidence consistently supports that improving routing robustness is a practical way to make sparse MoE capacity more reliable, not merely more balanced.

\section{Related work}
\label{related-work}

In this section, we discuss the differences among existing MoE training paradigms. We first review methods designed to improve load balancing in MoE training. We then provide an overview of distributionally robust optimization (DRO), which forms the foundation of DRMoET.


\textbf{Mixture-of-Experts (MoE) Training.}
Sparsely-gated MoE layers \citep{shazeer2017outrageously} scale language models by routing each token to a few feed-forward \emph{experts}. \emph{Load-balancing} objectives enhance training stability and encourage uniform expert utilization.  GShard \citep{lepikhin2021} adds a KL-style auxiliary loss with strict capacity limits, whereas Switch Transformer \citep{fedus2022switch} simplifies routing to top-1 selection, reducing both communication and instability.  
Later work formulates routing as balanced assignment: BASE Layers \citep{lewis2021base} approximately solve a linear assignment, \citep{wang2024auxiliary} adds a learned bias vector to the router logits, and Expert-Choice routing lets experts select tokens \citep{zhou2022mixture}. \citet{wu2024multi} introduce \emph{Multi-Head MoE (MH-MoE)}, which splits each input token into $h$ learned sub-tokens before routing. These methods make sparse MoE training practical, but primarily control \emph{which} experts receive tokens. DRMoET is orthogonal: it leaves routing intact and improves the quality of experts that may be selected under plausible but suboptimal routing.

\textbf{Distributionally Robust Optimization (DRO).}
DRO minimizes the \emph{worst-case} loss over an uncertainty set~$\mathcal U$ around the empirical distribution.  $f$-divergence balls yield variance-regularized objectives \citep{duchi2019variance}, while CVaR arises as an $\alpha$-quantile special case.  Group DRO \citep{sagawa2019distributionally} targets the hardest predefined demographic or topic group, and topic-level CVaR for language modeling \citep{oren2019distributionally} lowers perplexity on under-represented topics. Streaming losses can stabilize high-variance group estimates across batches \citep{wen2022distributionally}. DRMoET transfers this principle to MoE training, where the groups are not external labels but layer-specific expert outcomes induced by top-$k$ routing and measured with activation-weighted attribution.

\section{Preliminaries}
\label{preliminaries}

In this section we review the core elements of sparse MoE architectures, the challenge of expert load imbalance, balancing strategies and extensions that motivate our robust training framework.

\subsection{MoE Architectures}
A MoE layer replaces a single dense feedforward network with a collection of $n$ parallel expert subnetworks $\{\mathrm{FFN}_i\}_{i=1}^n$. For each token embedding $\mathbf{x}\in\mathbb{R}^d$, a router produces scores over experts. Let $\phi(\mathbf{x})\in\mathbb{R}^n$ denote the router logits, where $\phi_i(\mathbf{x})$ is the scalar score for expert $i$. With optional Gumbel noise $\mathbf{g}$ and temperature $\tau>0$, define the selected expert set and sparse routing weight as
\[
\mathcal{T}(\mathbf{x}) = \mathrm{TopK}_k\bigl(\phi(\mathbf{x})+\mathbf{g}\bigr),
\qquad
F_i(\mathbf{x})
= \frac{\exp((\phi_i(\mathbf{x})+g_i)/\tau)\,\mathbf{1}[i\in\mathcal{T}(\mathbf{x})]}
{\sum_{r\in\mathcal{T}(\mathbf{x})}\exp((\phi_r(\mathbf{x})+g_r)/\tau)},
\]
where $g_i=-\ln(-\ln u_i)$ with $u_i\sim\mathcal{U}(0,1)$ when noisy routing is used. Thus $F_i(\mathbf{x})=0$ for unselected experts and the MoE layer output is
\[
\mathbf{o}(\mathbf{x})
\;=\;
\sum_{i\in\mathcal{T}(\mathbf{x})} F_{i}(\mathbf{x})\,\mathrm{FFN}_i(\mathbf{x}).
\]

Let $t$ index the minibatch or training iteration, and let $\mathcal{X}^{(t)}$ be the multiset of all tokens in the minibatch with $N_t=|\mathcal{X}^{(t)}|$. We distinguish soft routing mass from hard selected-token frequency:
\[
P_i^{(t)} \;=\; \frac{1}{N_t}\sum_{\mathbf{x}\in\mathcal{X}^{(t)}} F_i(\mathbf{x}),
\qquad
f_i^{(t)} \;=\; \frac{1}{N_t}\sum_{\mathbf{x}\in\mathcal{X}^{(t)}} \mathbf{1}\!\left[i\in\mathcal{T}(\mathbf{x})\right].
\]
Here $N_tP_i^{(t)}$ is the expected routing-probability mass assigned to expert $i$, whereas $N_tf_i^{(t)}$ is the discrete number of tokens for which expert $i$ is selected.

\subsection{Expert Load Imbalance and Balancing Strategies}
In practice, the router may concentrate most tokens on a few experts, leaving the majority under-utilized. This \emph{load imbalance} degrades both parameter efficiency and model performance, as rarely used experts receive insufficient gradient signals, and heavily used experts become bottlenecks.

A standard remedy is an \emph{auxiliary load-balancing loss}~\citep{fedus2022switch,lepikhin2021}, defined over the $t$-th minibatch in the common form
\[
  L_{\mathrm{aux}}^{(t)}
  \;=\;
  n\sum_{i=1}^n f_i^{(t)}P_i^{(t)},
  \qquad
  L^{(t)} = L_{\mathrm{task}}^{(t)} + \lambda L_{\mathrm{aux}}^{(t)},
\]
where $f_i^{(t)}$ is the hard selected-token fraction, $P_i^{(t)}$ is the average router probability mass, and $\lambda>0$ controls the strength of balancing. Increasing $\lambda$ encourages the routing statistics to be more uniform, but can harm the primary task when balancing gradients conflict with specialization.

To avoid such interference, \citet{wang2024auxiliary} propose \emph{Loss-Free Balancing}, which drops the auxiliary load-balancing loss. Instead, it learns a per-expert bias vector $\mathbf{b}\in\mathbb{R}^n$, adding $b_i$ to the router logit $\phi_i(\mathbf{x})$ before top-$k$ selection. After each batch, $b_i$ is updated by
\[
b_i \;\leftarrow\; b_i + u\,\mathrm{sign}(\bar n - n_i),
\]
where $n_i$ is the hard token count for expert $i$, $\bar n$ its average, and $u$ a small update rate. Thus overloaded experts are depressed and underloaded experts are elevated, dynamically steering the router toward balanced loads without any auxiliary gradient.

\section{Method}
\label{method}

While load balancing equalizes traffic, it can (i) interfere with the task objective and (ii) leave the model brittle when plausible but non-optimal experts are selected. We propose \textbf{DRMoET}, a \emph{distributionally robust} training procedure that optimizes layer-wise worst expert-attributed loss \emph{without adding additional auxiliary losses or changing the routing rule}. The key distinction from standard DRO is that the robustness groups are created by the MoE computation itself: each expert in each layer receives loss attribution through the routed tokens and activations it actually processes. This turns routing robustness into a task-loss reweighting problem, preserving the router's natural specialization while directing stronger signal to high-loss routing outcomes.

\subsection{Distributionally Robust MoE Training}
\label{sec:dro_moe}

\textbf{Experts as DRO groups.}
Consider a model $f_\theta$ with $L$ MoE layers, each containing $E$ experts and using top-$k$ routing. For each layer $l$, we introduce an adversarial weight vector $\mu_{l,\cdot}\in\Delta_E$, where $\Delta_E=\{\mu\in\mathbb{R}^E_{\ge0}:\sum_{i=1}^E\mu_i=1\}$. Let $R_{l,i}(\theta)$ denote the population risk attributed to expert $i$ at layer $l$. We study the layer-wise DRO saddle problem
\begin{equation}
\label{eq:saddle_main}
\min_{\theta}\;\max_{\mu\in(\Delta_E)^L}
\;F(\theta,\mu)
\;\triangleq\;
\sum_{l=1}^L\sum_{i=1}^E \mu_{l,i}\,R_{l,i}(\theta),
\end{equation}
where $(\Delta_E)^L$ is the product of $L$ simplices. Since $F(\theta,\mu)$ is linear in each $\mu_{l,\cdot}$, the inner maximization equals $\sum_l\max_i R_{l,i}(\theta)$, directly reducing the layer-wise worst expert-attributed risk.

\textbf{Activation-weighted loss attribution.}
At iteration $t$, for a minibatch $\mathcal{B}$ containing $N_t$ tokens, compute token losses $\ell^{(b)}_j(\theta)=\ell(f_\theta(x^{(b)}_j),y^{(b)}_j)$. For each MoE layer $l$, the gate selects top-$k$ experts $\mathcal{T}^{(b)}_{j,l}\subset[E]$ with routing probabilities $p_{l,i}(x^{(b)}_j)=\mathrm{softmax}_{i\in\mathcal{T}^{(b)}_{j,l}}(r_{l,i}(x^{(b)}_j))$. Let $h_{l,i}(x^{(b)}_j)$ be the pre-merge output activation of expert $i$. We define the detached activation-weighted credit
\begin{equation}
\label{eq:credit_main}
\tilde c_{l,i}(x^{(b)}_j)
\;\triangleq\;
\sg\!\left(p_{l,i}(x^{(b)}_j)\cdot \big\|h_{l,i}(x^{(b)}_j)\big\|_2\right),
\end{equation}
where $\sg(\cdot)$ denotes stop-gradient. The credit incorporates both routing probability and activation magnitude while preventing the attribution coefficient itself from becoming an additional router or activation-norm objective.

\begin{algorithm}[ht]
  \caption{Distributionally Robust Mixture-of-Experts Training}
  \label{alg:moe-dro}
  \begin{algorithmic}[1]
  \REQUIRE Model $f_\theta$ with $L$ MoE layers, $E$ experts per layer, top-$k$ routing
  \REQUIRE EMA decay $\beta$, dual step-size schedule $\eta_t$
  \ENSURE Learned parameters $\theta$
  \STATE Initialize $\mu_{l,i}^{(0)}\gets1/E$, $\hat{\ell}_{l,i}^{(0)}\gets0$ for all $l\in[L]$, $i\in[E]$
  \FOR{$t = 1$ to $T$}
    \FOR{minibatch $\mathcal{B} = \{(x^{(b)}, y^{(b)})\}_{b=1}^B$ with $N_t$ tokens}
      \STATE Compute token losses $\ell^{(b)}_j = \ell(f_\theta(x^{(b)}_j), y^{(b)}_j)$
      \STATE Collect gate logits $r_{l,i}(x^{(b)}_j)$ and expert outputs $h_{l,i}(x^{(b)}_j)$
      \FOR{each MoE layer $l = 1, \ldots, L$}
        \STATE $\mathcal{T}^{(b)}_{j,l} \gets \text{top-}k$ experts for token $j$ based on $r_{l,\cdot}(x^{(b)}_j)$
        \STATE $p_{l,i}(x^{(b)}_j) \gets \text{softmax}_{i \in \mathcal{T}^{(b)}_{j,l}}(r_{l,i}(x^{(b)}_j))$
        \STATE $\tilde c_{l,i}(x^{(b)}_j) \gets \sg\!\left(p_{l,i}(x^{(b)}_j)\|h_{l,i}(x^{(b)}_j)\|_2\right)$
        \STATE Compute $\ell^{(t)}_{l,i} \gets \frac{1}{N_t}\sum_{b,j}\tilde c_{l,i}(x^{(b)}_j)\ell^{(b)}_j\mathbf{1}[i\in\mathcal{T}^{(b)}_{j,l}]$ for all $i$
        \STATE $\hat{\ell}^{(t)}_{l,\cdot} \gets \beta\hat{\ell}^{(t-1)}_{l,\cdot} + (1 - \beta)\sg(\ell^{(t)}_{l,\cdot})$
        \STATE $\mu_{l,\cdot}^{(t)} \gets \softmax\!\left(\mu_{l,\cdot}^{(t-1)} + \eta_t\hat{\ell}^{(t)}_{l,\cdot}\right)$
      \ENDFOR
      \STATE $\mathcal{L}_{\text{DRO}} \gets \sum_{l=1}^L \sum_{i=1}^E \sg(\mu_{l,i}^{(t)})\ell^{(t)}_{l,i}$
      \STATE Backpropagate $\mathcal{L}_{\text{DRO}}$ and update $\theta$
    \ENDFOR
  \ENDFOR
  \end{algorithmic}
\end{algorithm}

We attribute a normalized loss mass to expert $(l,i)$:
\begin{equation}
\label{eq:elli_main}
\ell^{(t)}_{l,i}
\;\triangleq\;
\frac{1}{N_t}\sum_{(b,j)\in\mathcal{B}}
\tilde c_{l,i}(x^{(b)}_j)\cdot \ell^{(b)}_j(\theta)\cdot \mathbf{1}\!\left[i\in\mathcal{T}^{(b)}_{j,l}\right].
\end{equation}
Normalizing by token count keeps the DRO term on a stable scale across micro-batch sizes and avoids changing the relative weighting of the standard task and router-regularization terms used by the training recipe. The population risk is $R_{l,i}(\theta)\triangleq\mathbb{E}[\mathcal{A}_{l,i}(x,y;\theta)]$, where $\mathcal{A}_{l,i}$ is the corresponding single-example, token-normalized analogue.

\textbf{Primal--dual optimization.}
We maintain an EMA of attributed losses,
\begin{equation}
\label{eq:ema_main}
\hat{\ell}^{(t)}_{l,i}
\;=\;
\beta\,\hat{\ell}^{(t-1)}_{l,i} + (1-\beta)\,\sg(\ell^{(t)}_{l,i}),
\qquad \beta\in(0,1),
\end{equation}
and update the dual variables by the entropy-regularized softmax rule used in the implementation:
\begin{equation}
\label{eq:mu_update_main}
\mu^{(t)}_{l,\cdot}
\;=\;
\softmax\!\left(\mu^{(t-1)}_{l,\cdot}+\eta_t\,\hat{\ell}^{(t)}_{l,\cdot}\right).
\end{equation}
This probability-space softmax rule is analyzed in Appendix~\ref{sec:DRMoET-theory}. There, we show that the implemented update can be viewed as entropic mirror ascent on a regularized robust objective, rather than as the log-space multiplicative-weights update $\softmax(\log\mu+\eta\hat\ell)$. Larger EMA losses increase the relative dual mass assigned to the corresponding experts while the entropy term keeps the update smooth. Experiments use the constant parameterization $\eta_t=\eta_0/\sqrt{E}$, reported as $\eta$. The primal objective at step $t$ is
\begin{equation}
\label{eq:ldro_main}
\mathcal{L}^{(t)}_{\mathrm{DRO}}
\;\triangleq\;
\sum_{l=1}^L\sum_{i=1}^E \sg(\mu^{(t)}_{l,i})\,\ell^{(t)}_{l,i},
\end{equation}
after the EMA and dual weights have been updated outside autograd from the current detached attribution statistics. Thus DRMoET operates through normalized task-loss reweighting with detached attribution coefficients, requiring only minimal code changes with negligible overhead.

\subsection{Convergence Guarantees}
\label{sec:convergence}

We analyze Algorithm~\ref{alg:moe-dro} for the implemented probability-space
softmax update
\[
\mu^{(t+1)}_{l}
=
\softmax\!\left(\mu^{(t)}_{l}+\eta_t\hat{\ell}^{(t)}_{l}\right).
\]
This update is not the log-space multiplicative-weights rule
$\softmax(\log\mu^{(t)}_l+\eta_t\hat{\ell}^{(t)}_l)$. Instead, as shown in
Appendix~\ref{sec:DRMoET-theory}, for the constant-dual-step setting
$\eta_t=\eta$, it is exactly entropic mirror ascent on a regularized robust
objective. Define
\[
\Omega(\mu_l)
=
-\sum_{i=1}^E\mu_{l,i}\log\mu_{l,i}
+
\frac12\|\mu_l\|_2^2,
\qquad
F_\eta(\theta,\mu)
=
F(\theta,\mu)
+
\frac1\eta\sum_{l=1}^L\Omega(\mu_l),
\]
and
\[
\Phi_\eta(\theta)
=
\max_{\mu\in(\Delta_E)^L}F_\eta(\theta,\mu).
\]

\begin{theorem}[Convergence to stationarity of regularized DRMoET]
\label{thm:regularized-main}
Under the assumptions in Appendix~\ref{sec:DRMoET-theory}, let
$\gamma_t=\gamma_0/\sqrt{t}$ and sample
$\tau$ with probability
$\Pr(\tau=t)=\gamma_t/\sum_{s=1}^T\gamma_s$. Then Algorithm~\ref{alg:moe-dro}
satisfies
\[
\mathbb{E}\!\left[
\|\nabla\Phi_\eta(\theta^{(\tau)})\|_2^2
\right]
\le
\widetilde{\mathcal{O}}(T^{-1/2})
+
\mathcal{O}\!\left(
\varepsilon_{\mathrm{bias}}^2
+
L_{\theta\mu}^2L^2E^2\eta^2
\varepsilon_{\mathrm{ema},T}^2
\right).
\]
In the idealized population-loss setting, where
$\varepsilon_{\mathrm{bias}}=0$ and
$\varepsilon_{\mathrm{ema},T}\to0$, the expected stationarity measure
vanishes at rate $\widetilde{\mathcal{O}}(T^{-1/2})$.
\end{theorem}

The guarantee is for the entropy-regularized robust objective
$\Phi_\eta$, which approximates the original hard worst-expert objective
$\Phi(\theta)=\sum_l\max_i R_{l,i}(\theta)$ with explicit bias
\[
\Phi(\theta)
\le
\Phi_\eta(\theta)-\frac{L}{2\eta}
\le
\Phi(\theta)+\frac{L\log E}{\eta}.
\]
Thus DRMoET has a convergence guarantee for a smooth regularized
approximation to the layer-wise DRO objective, while the approximation error
to the hard max is controlled by $L\log E/\eta$.

\section{Experiments}
\label{sec:experiments}

We evaluate DRMoET on large-scale, from-scratch pretraining at two model scales. All architectures, data pipelines, and evaluation protocols follow FLAME-MoE~\citep{kang2025flame} to ensure direct comparability. Beyond downstream accuracy, we conduct targeted analyses to validate the routing-robustness mechanism: expert-level loss distributions, forced-misrouting probes, expert-tier/OOD checks, functional specialization, and compact hyperparameter ablations; training-throughput results are deferred to the appendix.

\subsection{Pretraining Setup}
\label{sec:exp-setup}

\textbf{Data.}
All models are trained on data sampled from DataComp-LM (DCLM)~\citep{li2024datacomp}, following the FLAME-MoE data pipeline with the standard next-token prediction objective. The \texttt{FLAME-MoE-290M-746M} experiments use approximately 33.5B tokens. The main \texttt{FLAME-MoE-1.7B-10.3B} comparison uses approximately 67B tokens. Within each model scale and token budget, the compared methods use the same architecture and data pipeline.

\textbf{Model.}
We train two configurations from scratch. The \texttt{FLAME-MoE-290M-746M} model is a decoder-only Transformer with $E{=}32$ experts per MoE layer and top-$k$ routing with $k{=}8$, of which 2 are shared (always active) experts~\citep{kang2025flame}. The \texttt{FLAME-MoE-1.7B-10.3B} model uses $E{=}64$ experts to test whether gains persist at higher scale.

\textbf{Optimization.}
We adopt the FLAME-MoE training recipe: Adam optimizer with maximum learning rate $3\times10^{-4}$, global batch size 1024, sequence length 2048, and a Warmup--Stable--Decay (WSD) schedule with warmup ratio 0.01 and decay ratio 0.1. The global batch is formed from per-GPU micro-batches; the 290M runs use micro-batch size 16, while the 1.7B runs use micro-batch size 8 because of memory constraints. Both the baseline and DRMoET use the same auxiliary load-balancing loss and z-loss coefficients of $1\times10^{-3}$.

\textbf{Baselines.}
At each scale, we compare against (i) standard FLAME-MoE training and (ii) auxiliary-loss-free balancing baselines~\citep{wang2024auxiliary} using the Megatron-LM implementation. For the 290M model, we report $u\in\{10^{-3},10^{-2}\}$. At the 1.7B-active scale and 67B-token budget, we report the $u{=}10^{-2}$ setting, which is independently tuned with a larger batch size.

\subsection{Evaluation}
\label{sec:eval}

We report 0-shot performance following FLAME-MoE~\citep{kang2025flame}; ReCoRD is reported as F1, all others as accuracy. We evaluate on ARC-E / ARC-C, HellaSwag, PIQA, WinoGrande, SciQ, and ReCoRD. 

\textbf{ARC-Easy / ARC-Challenge}~\citep{clark2018think} evaluate grade-school science knowledge and deeper scientific reasoning, respectively.
\textbf{HellaSwag}~\citep{zellers2019hellaswag} tests grounded commonsense inference.
\textbf{PIQA}~\citep{bisk2020piqa} probes physical commonsense.
\textbf{WinoGrande}~\citep{sakaguchi2021winogrande} evaluates commonsense through pronoun resolution.
\textbf{SciQ}~\citep{welbl2017crowdsourcing} assesses scientific reasoning across biology, chemistry, and physics.
\textbf{ReCoRD}~\citep{zhang2018record} evaluates reading comprehension requiring entity reasoning over long contexts.

\subsection{Main Results}
\label{sec:main-results}

\begin{table*}[t]
\centering
\scriptsize
\setlength{\tabcolsep}{2pt}
\renewcommand{\arraystretch}{0.95}
\resizebox{\textwidth}{!}{%
\begin{tabular}{llccccccc|c}
\toprule
\textbf{Scale} & \textbf{Method} & \textbf{ARC-C} & \textbf{ARC-E} & \textbf{HellaSwag} & \textbf{PIQA} & \textbf{WinoGrande} & \textbf{SciQ} & \textbf{ReCoRD*} & \textbf{Average} \\
\midrule

\rowcolor{smallscalehead}
\multicolumn{2}{l}{\textbf{\texttt{290M-746M}}} & \multicolumn{8}{r}{Training tokens: 33.5B} \\
\rowcolor{smallscale}
& FLAME-MoE & 0.2329 & 0.5295 & 0.3447 & 0.6697 & 0.5012 & 0.8020 & 0.7023 & 0.5403 \\
\rowcolor{smallscale}
& Aux-free ($u{=}10^{-3}$) & 0.2261 & 0.4878 & 0.3283 & 0.6670 & 0.4949 & 0.7100 & 0.6445 & 0.5084 \\
\rowcolor{smallscale}
& Aux-free ($u{=}10^{-2}$) & 0.2184 & 0.4895 & 0.3301 & 0.6627 & 0.4949 & 0.6990 & 0.6440 & 0.5055 \\
\rowcolor{smallscale}
& DRMoET ($\eta{=}0.1$) & \textbf{0.2338} & 0.5522 & \textbf{0.3472} & 0.6844 & \textbf{0.5209} & 0.7950 & \textbf{0.7115} & 0.5493 \\
\rowcolor{smallscale}
& DRMoET ($\eta{=}0.01$) & 0.2270 & 0.5400 & 0.3470 & 0.6839 & 0.4941 & 0.8180 & 0.7007 & 0.5444 \\
\rowcolor{smallscale}
& DRMoET ($\eta{=}0.001$) & 0.2295 & \textbf{0.5682} & 0.3452 & 0.6888 & 0.5051 & \textbf{0.8190} & 0.7105 & \textbf{0.5523} \\
\rowcolor{smallscale}
& DRMoET ($\eta{=}0.0001$) & 0.2193 & 0.5556 & 0.3461 & \textbf{0.6970} & 0.5162 & 0.8060 & 0.7037 & 0.5491 \\

\midrule
\rowcolor{largescalehead}
\multicolumn{2}{l}{\textbf{\texttt{1.7B-10.3B}}} & \multicolumn{8}{r}{Training tokens: 67B} \\
\rowcolor{largescale}
& FLAME-MoE & 0.3481 & 0.6970 & 0.4942 & 0.7563 & 0.5904 & 0.9050 & \textbf{0.8467} & 0.6625 \\
\rowcolor{largescale}
& Aux-free ($u{=}10^{-2}$) & 0.3148 & 0.6843 & 0.4667 & 0.7459 & 0.5722 & 0.8930 & 0.8248 & 0.6431 \\
\rowcolor{largescale}
& DRMoET ($\eta{=}0.01$) & \textbf{0.3805} & \textbf{0.7210} & 0.4973 & \textbf{0.7666} & \textbf{0.6235} & 0.9020 & 0.8457 & \textbf{0.6767} \\
\rowcolor{largescale}
& DRMoET ($\eta{=}0.001$) & 0.3652 & 0.7024 & \textbf{0.4996} & 0.7628 & 0.6148 & \textbf{0.9110} & 0.8452 & 0.6716 \\

\bottomrule
\end{tabular}%
}
\caption{\textbf{Main results at two scales.} The two blocks use approximately 33.5B and 67B training tokens, respectively. All DRMoET rows use $\beta=0.999$. Averages are over the seven displayed tasks. Best per-task and average within each block are in \textbf{bold}. *ReCoRD is reported as F1.}
\label{tab:main_results}
\end{table*}

Tab.~\ref{tab:main_results} summarizes downstream performance at both model scales. At \texttt{290M-746M}, DRMoET improves over FLAME-MoE for every $\eta$ setting, with the best average ($\eta=0.001$) increasing from 0.5403 to 0.5523. The sweep shows predictable task tradeoffs: larger $\eta$ favors ReCoRD and WinoGrande, while $\eta=0.001$ gives the strongest average and the best ARC-E/SciQ scores. 

\Needspace{15\baselineskip}
\begin{wraptable}{r}{0.48\textwidth}
\centering
\small
\setlength{\tabcolsep}{3pt}
\renewcommand{\arraystretch}{1.0}
\begin{tabular}{@{}lccc@{}}
\toprule
\textbf{Metric} & \textbf{Baseline} & \textbf{DRMoET} & \textbf{$\Delta$ (\%)} \\
\midrule
Worst loss & 2.6930 & \textbf{2.6379} & $-2.05$ \\
Best loss  & 2.1149 & \textbf{2.0722} & $-2.02$ \\
Mean loss  & 2.3712 & \textbf{2.3699} & $-0.05$ \\
Range      & 0.5781 & \textbf{0.5656} & $-2.16$ \\
Std.       & 0.1516 & \textbf{0.1386} & $-8.58$ \\
CV         & 0.0639 & \textbf{0.0585} & $-8.45$ \\
\bottomrule
\end{tabular}
\caption{Expert-loss statistics at convergence.}
\label{tab:robustness_combined}
\end{wraptable}
Gains persist at \texttt{1.7B-10.3B} under the 67B-token budget: DRMoET with $\eta=0.01$ improves the seven-task average from 0.6625 to 0.6767, an absolute gain of 1.42 percentage points and a relative improvement of 2.14\%. This configuration improves five of seven tasks; $\eta=0.001$ reaches 0.6716 and improves six of seven tasks. FLAME-MoE retains the highest ReCoRD score.

\textbf{Comparison with auxiliary-loss-free balancing.}
At 290M-active scale, both auxiliary-loss-free baselines underperform FLAME-MoE baseline, as well as DRMoET results. This could be due to the small batch size we use during training. At the 1.7B-active scale, the 67B-token auxiliary-loss-free baseline with $u{=}10^{-2}$, independently tuned with a larger batch, reaches 0.6431. The result remains 1.94 points below FLAME-MoE and 3.36 points below the best DRMoET result.

\subsection{Routing Robustness Analysis}
\label{sec:robustness}

The preceding results establish empirical performance, we now shift to understand how DRMoET improves robustness to suboptimal routing.

\textbf{Expert-level loss distributions and forced-misrouting probe.}
We first examine whether DRMoET narrows the quality gap across routing outcomes on a 52M tokens validation set, and then directly test robustness by forcing suboptimal expert selection. For the expert-level analysis, we measure per-expert loss at convergence on the best 290M model.

As shown in Tab.~\ref{tab:robustness_combined}, DRMoET reduces the worst-expert loss by 2.05\% while keeping mean expert loss nearly unchanged ($-0.05\%$). The loss standard deviation and coefficient of variation decrease by over 8\%, indicating substantially more uniform expert quality.

\Needspace{10\baselineskip}
\begin{wraptable}{r}{0.48\textwidth}
\centering
\small
\setlength{\tabcolsep}{3pt}
\renewcommand{\arraystretch}{1.0}
\begin{tabular}{@{}lccc@{}}
\toprule
\textbf{Model} & \textbf{Normal} & \textbf{Misrouting} & \textbf{Deg.} \\
\midrule
Baseline & 3.187 & 7.249 & $2.27\times$ \\
DRMoET & \textbf{3.182} & \textbf{7.071} & \textbf{$2.22\times$} \\
\bottomrule
\end{tabular}
\caption{Forced mid-$k$ misrouting.}
\label{tab:misrouting}
\end{wraptable}
For the misrouting probe, we replace each token's top-$k$ experts with its middle-$k$ experts (ranked by routing probability), simulating a realistic scenario where the router makes plausible but non-optimal selections, and measure per-token loss across six evaluation domains.

Tab.~\ref{tab:misrouting} directly validates the core claim: under forced misrouting, DRMoET exhibits 4.3\% less loss degradation than the baseline, confirming that the model is more robust when routing is imperfect.

\textbf{Expert tier analysis and out-of-distribution robustness.}
To understand which experts benefit most, we rank all experts by total routing traffic and partition them into top, mid, and bottom tiers. We further evaluate expert loss uniformity under distribution shift using 60,000 OOD examples (4.2M tokens) from six benchmark families not seen during training.

Tab.~\ref{tab:tier_ood_combined}(a) shows that mid-tier experts improve the most, consistent with DRMoET strengthening the experts that matter most under plausible but non-optimal routing. Tab.~\ref{tab:tier_ood_combined}(b) extends this to unseen data: mean loss is nearly unchanged, while the loss range, standard deviation, and coefficient of variation (CV=std./mean) decrease, indicating less dispersed expert losses.

\begin{table}[!htbp]
\centering
\scriptsize
\setlength{\tabcolsep}{2.5pt}
\renewcommand{\arraystretch}{0.95}
\begin{minipage}[t]{0.46\linewidth}
\centering
\begin{tabular}{@{}lccc@{}}
\toprule
\textbf{Tier} & \textbf{FLAME-MoE} & \textbf{DRMoET} & \textbf{Improvement(\%)} \\
\midrule
Top & 3.176 & 3.188 & $-0.4\%$ \\
Mid & 3.100 & 3.077 & $+0.7\%$ \\
Bottom & 3.072 & 3.058 & $+0.4\%$ \\
\bottomrule
\end{tabular}

\vspace{2pt}
(a) Per-tier loss.
\end{minipage}\hfill
\begin{minipage}[t]{0.40\linewidth}
\centering
\begin{tabular}{@{}lccc@{}}
\toprule
\textbf{Metric} & \textbf{FLAME-MoE} & \textbf{DRMoET} & \textbf{$\Delta (\%)$} \\
\midrule
Mean loss & 3.151 & 3.157 & $+0.2\%$ \\
Range & 1.018 & 0.947 & $-7.0\%$ \\
Std. & 0.241 & 0.236 & $-2.1\%$ \\
CV & 0.076 & 0.075 & $-1.3\%$ \\
\bottomrule
\end{tabular}

\vspace{2pt}
(b) OOD uniformity.
\end{minipage}
\caption{\textbf{Expert tier and OOD analysis.} Mid-tier experts improve the most, and under distribution shift DRMoET yields more uniform expert quality while keeping mean loss nearly unchanged.}
\label{tab:tier_ood_combined}
\end{table}

\textbf{Functional specialization.}
A natural concern is whether improved routing robustness reduces expert specialization. We test this with 10k examples per benchmark, grouped into six functional domains, and compute specialization metrics via teacher-forced inference.

Here Max Selectivity is each expert's largest domain-selection share, $\Delta S$ is the gap between its top and second domain-selection shares, $\Delta Q$ is the competence advantage on its preferred domain, and $I(E;G)$ is mutual information between selected expert $E$ and functional domain $G$. All four metrics improve under DRMoET (Tab.~\ref{tab:specialization}), indicating stronger domain--expert association and better performance on each expert's specialized domains.

\subsection{Ablation Studies}
\label{sec:ablations}

\textbf{Activation-weighted vs.\ raw-probability credit.}
We compare activation-weighted credit (Eq.~\ref{eq:credit_main}) against a variant using only raw routing probability. Activation-weighted credit improves the best average from 0.6138 to 0.6263 at matched $\eta$ (Tab.~\ref{tab:ablation_combined}a), supporting attribution that reflects both selection and response magnitude. This is also consistent with the MoE forward computation: routing probability scales the expert activation, so credit should reflect the routed activation rather than probability alone.

\textbf{Top-$k$ routing sparsity and EMA decay $\beta$.}
We ablate routing sparsity by increasing from top-6 to top-12 (excluding shared experts), and sweep $\beta$ at fixed $\eta=0.001$.

Under denser routing (Tab.~\ref{tab:ablation_combined}b), DRMoET still improves, though the gain is smaller because larger $k$ reduces sensitivity to any single expert choice. In the $\beta$ sweep (Tab.~\ref{tab:ablation_combined}c), $\beta{=}0.9$ is strongest on ARC-E and HellaSwag, $\beta{=}0.99$ on PIQA, and $\beta{=}0.999$ on SciQ and ReCoRD while also giving the best average. We therefore use $\beta{=}0.999$ as the default.

\Needspace{15\baselineskip}
\begin{wraptable}{r}{0.48\textwidth}
\centering
\small
\setlength{\tabcolsep}{3pt}
\begin{tabular}{@{}lr@{}}
\toprule
\textbf{Metric} & \makecell[r]{\textbf{Improvement}\\\textbf{(\%)}} \\
\midrule
Avg Max Selectivity & $+1.0\%$ \\
\makecell[l]{Avg $\Delta S$\\(specialization margin)} & $+1.2\%$ \\
\makecell[l]{Avg $\Delta Q$\\(competence advantage)} & $+30.2\%$ \\
Avg MI $I(E;G)$ & $+12.9\%$ \\
\bottomrule
\end{tabular}
\caption{Domain--expert specialization metrics. DRMoET improves all measures.}
\label{tab:specialization}
\end{wraptable}
\textbf{Interaction with load balancing.}
Removing the standard load-balancing loss from DRMoET while retaining the remaining 290M training configuration yields a seven-task average of 0.5448, compared with 0.5403 for standard FLAME-MoE and 0.5523 for DRMoET with balancing (Appendix~\ref{app:load-balancing-ablation}). The combined configuration performs best, supporting the use of DRMoET as a complement to balancing. This comparison does not include FLAME-MoE without balancing, so it does not isolate DRMoET's incremental effect within a no-balancing recipe.

\begin{table}[!t]
\centering
\scriptsize
\setlength{\tabcolsep}{3.0pt}
\renewcommand{\arraystretch}{0.95}
\begin{tabular}{@{}llcccccc@{}}
\toprule
\textbf{Ablation} & \textbf{Setting} & \textbf{ARC-E} & \textbf{HellaSwag} & \textbf{PIQA} & \textbf{SciQ} & \textbf{ReCoRD*} & \textbf{Average} \\
\midrule
\multirow{5}{*}{Credit}
& FLAME-MoE & 0.5295 & 0.3447 & 0.6697 & 0.8020 & 0.7023 & 0.6096 \\
& Activation weighted, $\eta=10^{-2}$ & 0.5400 & \textbf{0.3470} & 0.6839 & 0.8180 & 0.7007 & 0.6179 \\
& Activation weighted, $\eta=10^{-3}$ & \textbf{0.5682} & 0.3452 & \textbf{0.6888} & \textbf{0.8190} & \textbf{0.7105} & \textbf{0.6263} \\
& Raw probability, $\eta=10^{-2}$ & 0.5455 & 0.3384 & 0.6839 & 0.7920 & 0.6938 & 0.6107 \\
& Raw probability, $\eta=10^{-3}$ & 0.5421 & 0.3435 & 0.6861 & 0.7910 & 0.7065 & 0.6138 \\
\midrule
\multirow{3}{*}{Top-12 routing}
& FLAME-MoE & 0.5459 & 0.3506 & \textbf{0.6991} & 0.7950 & 0.7033 & 0.6188 \\
& $\eta=10^{-3}$ & \textbf{0.5640} & \textbf{0.3523} & 0.6910 & \textbf{0.8130} & \textbf{0.7056} & \textbf{0.6252} \\
& $\eta=10^{-2}$ & 0.5619 & 0.3515 & \textbf{0.6991} & 0.8110 & 0.6998 & 0.6247 \\
\midrule
\multirow{5}{*}{EMA decay $\beta$}
& FLAME-MoE & 0.5295 & 0.3447 & 0.6697 & 0.8020 & 0.7023 & 0.6096 \\
& $\beta=0.9$ & 0.5568 & \textbf{0.3477} & 0.6915 & 0.7970 & 0.7048 & 0.6196 \\
& $\beta=0.99$ & 0.5408 & 0.3455 & \textbf{0.6937} & 0.7940 & 0.7022 & 0.6152 \\
& $\beta=0.999$ & \textbf{0.5682} & 0.3452 & 0.6888 & \textbf{0.8190} & \textbf{0.7105} & \textbf{0.6263} \\
& $\beta=1$ & 0.5467 & 0.3465 & 0.6844 & 0.7850 & 0.7068 & 0.6139 \\
\bottomrule
\end{tabular}
\caption{\textbf{Ablation studies.}
We ablate credit assignment, routing sparsity, and EMA smoothing. Averages are over the five displayed tasks within each ablation block. *ReCoRD is reported as F1.}
\label{tab:ablation_combined}
\end{table}

\subsection{Training throughput}
\label{sec:throughput}

\textbf{Computational complexity.}
Let $N_t$ be the number of tokens per step, $d$ the hidden size, and $d_{\mathrm{ff}}$ the expert FFN width. Per MoE layer, the selected expert FFNs require $\mathcal{O}(N_t k d d_{\mathrm{ff}})$ operations. DRMoET reuses their outputs and routing probabilities: computing activation norms costs $\mathcal{O}(N_t k d)$, sparse loss attribution costs $\mathcal{O}(N_t k)$, and the EMA and softmax updates cost $\mathcal{O}(E)$. Across $L$ layers, the additional arithmetic is therefore $\mathcal{O}(L N_t k d + L E)$, compared with $\mathcal{O}(L N_t k d d_{\mathrm{ff}})$ for the expert FFNs. The relative overhead is $\mathcal{O}(1/d_{\mathrm{ff}} + E/(N_t k d d_{\mathrm{ff}}))$, reducing to $\mathcal{O}(1/d_{\mathrm{ff}})$ when $E=\mathcal{O}(N_t k d)$. The persistent EMA and dual-weight state contains only $\mathcal{O}(LE)$ scalars. Since attribution and dual weights are detached, the objective reweights existing token losses and uses a single backward pass, with no additional expert forward passes. This operation count explains the small added computation; wall-clock effects are measured below.

\textbf{Measured throughput.}
A practical concern for distributionally robust MoE training is whether the additional robustness machinery introduces non-trivial runtime overhead. We therefore report end-to-end training throughput measured in TFLOP/s/GPU and iteration time (ms) under identical software settings and comparable hardware. As shown in Tab.~\ref{tab:throughput}, \emph{DRMoET does not incur a slowdown relative to the baseline}: iteration time is slightly lower and throughput is modestly higher, with lower variance as well. The observed speedup should be interpreted as a measured systems outcome, potentially reflecting reduced straggler variance rather than an algorithmic guarantee of faster training.

\begin{table}[h!]
\centering
\small
\begin{tabular}{lccc}
\toprule
\textbf{Metric} & \textbf{Baseline} & \textbf{DRMoET} & \textbf{$\Delta$} \\
\midrule
TFLOP/s/GPU & $64.17{\pm}1.81$ & $67.19{\pm}0.93$ & $+4.7\%$ \\
Time/iter (ms) & $6460{\pm}185$ & $6166{\pm}87$ & $-4.6\%$ \\
Throughput std & $1.81$ & $0.93$ & $2{\times}$ \\
\bottomrule
\end{tabular}
\caption{\textbf{Training throughput comparison.} Mean $\pm$ std across measured iterations. DRMoET introduces no slowdown and exhibits lower variance.}
\label{tab:throughput}
\end{table}

\FloatBarrier

\section{Discussion}

\paragraph{DRMoET improves routing robustness.}
Our analysis suggest that DRMoET improves MoE performance not simply by changing the average expert utilization, but by reducing the penalty of imperfect routing. The forced misrouting probe shows that when tokens are deliberately assigned to plausible but lower-ranked experts, DRMoET suffers less degradation than the baseline, indicating that non-top routing paths become more reliable. This is consistent with the expert-tier and loss-dispersion analyses: DRMoET narrows expert-quality gaps while preserving useful specialization, rather than collapsing all experts toward uniform behavior. These findings support the central motivation of DRMoET: routing robustness is complementary to load balancing, because a balanced router is not sufficient if the experts selected under uncertainty remain unevenly trained.

\paragraph{Limitations.} We evaluate DRMoET at two model scales and a fixed set of expert configurations, future work should test whether the same routing-robustness patterns hold across broader model families, expert counts, and routing depths. Our experiments focus on standard pretraining and benchmark settings; we have not yet explicitly constructed highly imbalanced pretraining mixtures or long-tail evaluation suites, where rare domains and skewed routing patterns may more directly stress expert robustness. Such settings could provide a sharper test of whether DRMoET improves reliability under distributional imbalance. Our forced-misrouting and expert-tier/OOD analyses are diagnostic probes on trained checkpoints; extending these measurements throughout training could provide a more detailed view of how expert robustness and specialization emerge.

\section*{Conclusion}

We introduced DRMoET, a layer-wise expert-level DRO objective for MoE training that improves robustness to suboptimal routing. The central lesson is that sparse scaling should not be judged only by balanced utilization: real routers are imperfect, and distribution shifts can send tokens through weaker, non-top experts. DRMoET addresses this reliability gap by optimizing high-loss expert-attributed outcomes while preserving the router and the standard sparse computation path. Empirically, it reduces expert-loss standard deviation by over 8\%, reduces excess loss under forced misrouting by 4.3\%, preserves and even strengthens specialization, and outperforms both standard FLAME-MoE and auxiliary-loss-free balancing at the \texttt{290M-746M} and \texttt{1.7B-10.3B} scales. At 67B training tokens, the best large-model result reaches 0.6767 versus 0.6625 for FLAME-MoE and 0.6431 for auxiliary-loss-free balancing, showing that the downstream advantage persists under longer training. These results suggest that routing robustness is a practical objective for making MoE capacity gains more reliable.

\clearpage

\bibliography{neurips_2026}
\bibliographystyle{plainnat}

\newpage
\appendix
\onecolumn

\section{Analysis of DRMoET}
\label{sec:DRMoET-theory}

We analyze the detached-credit version of DRMoET stated in
Algorithm~\ref{alg:moe-dro}. The analysis uses a differentiable relaxation
of the router, or equivalently holds locally within fixed top-$k$ routing
regions. This is the standard way to state a smooth nonconvex guarantee for
sparse routers; the exact hard top-$k$ map is discontinuous at routing ties.

The key point is that the implemented update
\[
\mu_l^{(t+1)}
=
\softmax\!\left(\mu_l^{(t)}+\eta\hat\ell_l^{(t)}\right)
\]
is not the log-space multiplicative-weights update
$\softmax(\log\mu_l^{(t)}+\eta\hat\ell_l^{(t)})$. Instead, it is exactly
entropic mirror ascent on a regularized robust objective. This section gives
the formal statement and convergence guarantee.

\subsection{Problem formulation}
\label{appendix:problem_formulation}

Let the training distribution over examples be $\mathcal{D}$. Each example
$(x,y)\sim\mathcal{D}$ contains tokens indexed by $j$. At iteration $t$ on a
minibatch $\mathcal{B}$ with $N_t$ tokens, define token losses
\[
\ell^{(b)}_j(\theta)
=
\ell\!\left(f_\theta(x^{(b)}_j),y^{(b)}_j\right).
\]
For an MoE layer $l$ with top-$k$ routing, let
$\mathcal{T}^{(b)}_{j,l}$ be the selected experts for token
$x^{(b)}_j$, and define routing probabilities restricted to the selected set
\[
p_{l,i}\!\left(x^{(b)}_j\right)=
\begin{cases}
\mathrm{softmax}_{r\in\mathcal{T}^{(b)}_{j,l}}
\!\big(r_{l,r}(x^{(b)}_j)\big)_i,
& i\in\mathcal{T}^{(b)}_{j,l},\\
0,
& i\notin\mathcal{T}^{(b)}_{j,l}.
\end{cases}
\]
Let $h_{l,i}(x^{(b)}_j)$ denote the pre-merge output activation of expert
$i$ at layer $l$. We use detached activation-weighted credit
\begin{equation}
\label{eq:credit_appendix}
\tilde c_{l,i}\!\left(x^{(b)}_j\right)
=
\operatorname{sg}\!\left(
p_{l,i}\!\left(x^{(b)}_j\right)
\big\|h_{l,i}(x^{(b)}_j)\big\|_2
\right),
\end{equation}
where $\operatorname{sg}(\cdot)$ denotes stop-gradient.

The normalized minibatch attribution is
\begin{equation}
\label{eq:ell_li_appendix}
\ell^{(t)}_{l,i}
=
\frac{1}{N_t}
\sum_{(b,j)\in\mathcal{B}}
\tilde c_{l,i}\!\left(x^{(b)}_j\right)
\ell^{(b)}_j(\theta)
\mathbf{1}\!\left[i\in\mathcal{T}^{(b)}_{j,l}\right].
\end{equation}
Define the corresponding population risks as
\[
R_{l,i}(\theta)
=
\mathbb{E}_{(x,y)\sim\mathcal{D}}
\big[\mathcal{A}_{l,i}(x,y;\theta)\big],
\]
where $\mathcal{A}_{l,i}$ is the single-example normalized analogue of
\eqref{eq:ell_li_appendix}. The unregularized layer-wise DRO objective is
\[
F(\theta,\mu)
=
\sum_{l=1}^L\sum_{i=1}^E
\mu_{l,i}R_{l,i}(\theta),
\qquad
\mu_l\in\Delta_E,
\]
and
\[
\Phi(\theta)
=
\max_{\mu\in(\Delta_E)^L}
F(\theta,\mu)
=
\sum_{l=1}^L\max_{i\in[E]}R_{l,i}(\theta).
\]

For the implemented softmax dual update, define
\[
\Omega(\mu_l)
=
H(\mu_l)+\frac12\|\mu_l\|_2^2
=
-\sum_{i=1}^E\mu_{l,i}\log\mu_{l,i}
+
\frac12\|\mu_l\|_2^2,
\]
and the regularized robust objective
\begin{equation}
\label{eq:regularized_objective_appendix}
F_\eta(\theta,\mu)
=
F(\theta,\mu)
+
\frac1\eta
\sum_{l=1}^L
\Omega(\mu_l),
\qquad
\Phi_\eta(\theta)
=
\max_{\mu\in(\Delta_E)^L}
F_\eta(\theta,\mu).
\end{equation}

\subsection{Standing assumptions}
\label{appendix:assumptions}

\begin{enumerate}[label=(A\arabic*)]
  \item \label{as:loss-range}
  \textbf{Bounded risks and attributions.}
  There exists $G>0$ such that
  $0\le \ell_{l,i}^{(t)}\le G$,
  $0\le \hat\ell_{l,i}^{(t)}\le G$, and
  $0\le R_{l,i}(\theta^{(t)})\le G$
  for all $l,i,t$ along the training trajectory.

  \item \label{as:smooth}
  \textbf{Smooth stop-gradient risks.}
  Each detached-credit risk $R_{l,i}(\theta)$ is
  $L_\theta$-smooth in $\theta$ under the differentiable routing relaxation.
  The regularized value function $\Phi_\eta$ is $L_\Phi$-smooth and bounded
  below by $\Phi_{\eta,\inf}$.

  \item \label{as:grad}
  \textbf{Stochastic primal gradients.}
  Let $g^{(t)}$ denote the stochastic gradient obtained by backpropagating
  $\sum_{l,i}\operatorname{sg}(\mu^{(t)}_{l,i})\ell^{(t)}_{l,i}$ with
  detached credits. For fixed $(\theta^{(t)},\mu^{(t)})$,
  \[
  \left\|
  \mathbb{E}\!\left[g^{(t)}\mid\theta^{(t)},\mu^{(t)}\right]
  -
  \nabla_\theta F(\theta^{(t)},\mu^{(t)})
  \right\|_2
  \le
  \varepsilon_{\mathrm{bias}},
  \]
  and
  \[
  \mathbb{E}\!\left[
  \|g^{(t)}\|_2^2
  \mid
  \theta^{(t)},\mu^{(t)}
  \right]
  \le
  M^2.
  \]

  \item \label{as:cross-lip}
  \textbf{Lipschitz dependence on the dual variable.}
  There exists $L_{\theta\mu}>0$ such that for all $\theta$ and
  $\mu,\nu\in(\Delta_E)^L$,
  \[
  \left\|
  \nabla_\theta F(\theta,\mu)
  -
  \nabla_\theta F(\theta,\nu)
  \right\|_2
  \le
  L_{\theta\mu}
  \sum_{l=1}^L\|\mu_l-\nu_l\|_1.
  \]

  \item \label{as:risk-lip}
  \textbf{Lipschitz risks.}
  There exists $L_R>0$ such that for every layer $l$,
  \[
  \|R_l(\theta)-R_l(\theta')\|_\infty
  \le
  L_R\|\theta-\theta'\|_2,
  \]
  where $R_l(\theta)=(R_{l,1}(\theta),\ldots,R_{l,E}(\theta))$.

  \item \label{as:ema-error}
  \textbf{EMA estimation error.}
  Define
  \[
  \xi_t
  =
  \max_{l\in[L]}
  \left\|
  \hat\ell_l^{(t)}
  -
  R_l(\theta^{(t)})
  \right\|_\infty .
  \]
  For $\gamma_t=\gamma_0/\sqrt{t}$, define the weighted average EMA error
  \[
  \varepsilon_{\mathrm{ema},T}^2
  =
  \frac{
  \sum_{t=1}^T
  \gamma_t
  \mathbb{E}[\xi_t^2]
  }{
  \sum_{t=1}^T
  \gamma_t
  }.
  \]
  This term captures both minibatch noise and EMA lag. It vanishes in the
  idealized population-loss setting or under a sufficiently accurate
  increasing-batch/variance-reduced estimate.

  \item \label{as:stepsizes}
  \textbf{Step sizes.}
  The primal update is
  \[
  \theta^{(t+1)}
  =
  \theta^{(t)}-\gamma_t g^{(t)},
  \qquad
  \gamma_t=\gamma_0/\sqrt{t}.
  \]
  The dual update uses the implemented softmax rule
  \[
  \mu_l^{(t+1)}
  =
  \softmax\!\left(\mu_l^{(t)}+\eta\hat\ell_l^{(t)}\right),
  \]
  with fixed $\eta=\eta_0/\sqrt{E}$.
\end{enumerate}

For each $\theta$, let
\[
\mu_\eta^\star(\theta)
\in
\arg\max_{\mu\in(\Delta_E)^L}
F_\eta(\theta,\mu)
\]
denote the regularized dual response.

\subsection{The implemented update as mirror ascent}
\label{appendix:implemented-update}

\begin{lemma}[Regularized mirror-ascent form]
\label{lem:regularized_mirror_appendix}
For fixed $\theta$ and layer $l$, one entropic mirror-ascent step on
$F_\eta$ with KL geometry and step size $\eta$ is
\[
\mu_l^{+}
=
\arg\max_{\nu\in\Delta_E}
\left\{
\eta
\left\langle
\nu,
\nabla_{\mu_l}F_\eta(\theta,\mu)
\right\rangle
-
D_{\mathrm{KL}}(\nu\|\mu_l)
\right\}.
\]
This update has the closed form
\[
\mu_l^{+}
=
\softmax\!\left(\mu_l+\eta R_l(\theta)\right).
\]
Replacing $R_l(\theta)$ by the EMA estimate $\hat\ell_l^{(t)}$ gives the
implemented update in Algorithm~\ref{alg:moe-dro}.
\end{lemma}

\begin{proof}
For the regularizer
\[
\Omega(\mu_l)
=
-\sum_i\mu_{l,i}\log\mu_{l,i}
+
\frac12\|\mu_l\|_2^2,
\]
we have
\[
\nabla_{\mu_l}\Omega(\mu_l)
=
\mu_l-\log\mu_l-\mathbf{1},
\]
where the logarithm is taken elementwise. Therefore
\[
\nabla_{\mu_l}F_\eta(\theta,\mu)
=
R_l(\theta)
+
\frac1\eta
\left(\mu_l-\log\mu_l-\mathbf{1}\right).
\]
The KKT conditions for the entropic mirror-ascent subproblem give
\[
\mu_{l,i}^{+}
\propto
\mu_{l,i}
\exp\!\left(
\eta R_{l,i}(\theta)
+
\mu_{l,i}
-
\log\mu_{l,i}
-
1
\right).
\]
The multiplicative base $\mu_{l,i}$ cancels the $-\log\mu_{l,i}$ term,
and the constant $-1$ is absorbed into normalization. Hence
\[
\mu_l^{+}
=
\softmax\!\left(\mu_l+\eta R_l(\theta)\right).
\]
\end{proof}

\subsection{Dual tracking}
\label{appendix:dual-tracking}

For each fixed $\theta$, define the population softmax response map
\[
T_{\eta,\theta,l}(\mu_l)
=
\softmax\!\left(\mu_l+\eta R_l(\theta)\right).
\]
The regularized maximizer $\mu^\star_{\eta,l}(\theta)$ is the unique fixed
point of this map:
\[
\mu^\star_{\eta,l}(\theta)
=
T_{\eta,\theta,l}
\!\left(\mu^\star_{\eta,l}(\theta)\right).
\]

\begin{lemma}[Softmax response is contractive]
\label{lem:softmax_contract_appendix}
For any $a,b\in\mathbb{R}^E$,
\[
\|\softmax(a)-\softmax(b)\|_\infty
\le
\frac12\|a-b\|_\infty .
\]
Consequently, for any $\theta,\theta'$ and any $\mu,\nu\in\Delta_E$,
\[
\left\|
T_{\eta,\theta,l}(\mu)
-
T_{\eta,\theta',l}(\nu)
\right\|_\infty
\le
\frac12\|\mu-\nu\|_\infty
+
\frac{\eta}{2}
\|R_l(\theta)-R_l(\theta')\|_\infty .
\]
\end{lemma}

\begin{proof}
Let $s=\softmax(z)$. The Jacobian of softmax satisfies
\[
\frac{\partial s_i}{\partial z_j}
=
s_i(\mathbf{1}[i=j]-s_j).
\]
For each row,
\[
\sum_j
\left|
\frac{\partial s_i}{\partial z_j}
\right|
=
2s_i(1-s_i)
\le
\frac12.
\]
Hence softmax is $1/2$-Lipschitz from $\ell_\infty$ to $\ell_\infty$.
Applying this to
$a=\mu+\eta R_l(\theta)$ and
$b=\nu+\eta R_l(\theta')$
gives the result.
\end{proof}

\begin{lemma}[Tracking of the regularized dual maximizer]
\label{lem:dual_tracking_appendix}
Define
\[
e_t
=
\max_{l\in[L]}
\left\|
\mu_l^{(t)}
-
\mu^\star_{\eta,l}(\theta^{(t)})
\right\|_\infty .
\]
Then the implemented update satisfies
\[
e_{t+1}
\le
\frac12 e_t
+
\frac{\eta}{2}\xi_t
+
\eta L_R\gamma_t\|g^{(t)}\|_2 .
\]
Consequently, for $\gamma_t=\gamma_0/\sqrt{t}$,
\[
\frac{
\sum_{t=1}^T
\gamma_t
\mathbb{E}[e_t^2]
}{
\sum_{t=1}^T
\gamma_t
}
\le
C
\left(
\frac{e_1^2}{\sum_{t=1}^T\gamma_t}
+
\eta^2\varepsilon_{\mathrm{ema},T}^2
+
\eta^2L_R^2M^2
\frac{\sum_{t=1}^T\gamma_t^3}{\sum_{t=1}^T\gamma_t}
\right)
\]
for a universal constant $C$.
\end{lemma}

\begin{proof}
Using the implemented update and the fixed-point identity for
$\mu^\star_{\eta,l}(\theta^{(t)})$,
\[
\begin{aligned}
&
\left\|
\mu_l^{(t+1)}
-
\mu^\star_{\eta,l}(\theta^{(t)})
\right\|_\infty
\\
&\le
\left\|
\softmax\!\left(
\mu_l^{(t)}+\eta\hat\ell_l^{(t)}
\right)
-
\softmax\!\left(
\mu^\star_{\eta,l}(\theta^{(t)})
+
\eta R_l(\theta^{(t)})
\right)
\right\|_\infty
\\
&\le
\frac12
\left\|
\mu_l^{(t)}
-
\mu^\star_{\eta,l}(\theta^{(t)})
\right\|_\infty
+
\frac{\eta}{2}
\left\|
\hat\ell_l^{(t)}
-
R_l(\theta^{(t)})
\right\|_\infty .
\end{aligned}
\]
It remains to account for the movement of the target
$\mu^\star_{\eta,l}(\theta)$ between $\theta^{(t)}$ and
$\theta^{(t+1)}$. By Lemma~\ref{lem:softmax_contract_appendix} and the
Lipschitzness of $R_l$,
\[
\left\|
\mu^\star_{\eta,l}(\theta^{(t+1)})
-
\mu^\star_{\eta,l}(\theta^{(t)})
\right\|_\infty
\le
\eta
\left\|
R_l(\theta^{(t+1)})
-
R_l(\theta^{(t)})
\right\|_\infty
\le
\eta L_R\gamma_t\|g^{(t)}\|_2 .
\]
Combining the two inequalities and maximizing over $l$ gives the one-step
recursion.

Unrolling the geometric recursion, squaring, using
$(a+b+c)^2\le 3(a^2+b^2+c^2)$, and applying
$\mathbb{E}\|g^{(t)}\|_2^2\le M^2$ gives
\[
\mathbb{E}[e_t^2]
\le
C
\left(
2^{-t}e_1^2
+
\eta^2
\sum_{s=1}^{t-1}
2^{-(t-1-s)}
\mathbb{E}[\xi_s^2]
+
\eta^2L_R^2
\sum_{s=1}^{t-1}
2^{-(t-1-s)}
\gamma_s^2M^2
\right).
\]
Multiplying by $\gamma_t$, summing over $t$, using the monotonicity of
$\gamma_t$, and dividing by $\sum_t\gamma_t$ yields the stated weighted
tracking bound.
\end{proof}

\subsection{Connection to the original hard DRO objective}
\label{appendix:regularization-bias}

\begin{lemma}[Approximation to the hard worst-expert objective]
\label{lem:regularization_bias_appendix}
Let
\[
\Phi(\theta)
=
\max_{\mu\in(\Delta_E)^L}F(\theta,\mu)
=
\sum_{l=1}^L\max_i R_{l,i}(\theta).
\]
Then, for all $\theta$,
\[
\Phi(\theta)
\le
\Phi_\eta(\theta)-\frac{L}{2\eta}
\le
\Phi(\theta)+\frac{L\log E}{\eta}.
\]
\end{lemma}

\begin{proof}
For each layer, if $e_{i^\star}$ is the one-hot vector on a maximizer of
$R_{l,i}(\theta)$, then
\[
\Omega(e_{i^\star})=\frac12.
\]
Therefore
\[
\Phi_\eta(\theta)
\ge
\Phi(\theta)+\frac{L}{2\eta}.
\]
Conversely, for every $\mu_l\in\Delta_E$,
\[
F(\theta,\mu)\le \Phi(\theta),
\qquad
\Omega(\mu_l)
=
H(\mu_l)+\frac12\|\mu_l\|_2^2
\le
\log E+\frac12.
\]
Thus
\[
\Phi_\eta(\theta)
\le
\Phi(\theta)
+
\frac{L}{\eta}
\left(\log E+\frac12\right).
\]
Subtracting $L/(2\eta)$ gives the result.
\end{proof}

\subsection{Convergence to stationarity}
\label{appendix:stationarity}

\begin{theorem}[Convergence to stationarity of regularized DRMoET]
\label{thm:regularized_convergence_appendix}
Under Assumptions~\ref{as:loss-range}--\ref{as:stepsizes},
Algorithm~\ref{alg:moe-dro} satisfies
\[
\frac{
\sum_{t=1}^T
\gamma_t
\mathbb{E}
\left[
\|\nabla\Phi_\eta(\theta^{(t)})\|_2^2
\right]
}{
\sum_{t=1}^T
\gamma_t
}
\le
\widetilde{\mathcal{O}}(T^{-1/2})
+
\mathcal{O}
\left(
\varepsilon_{\mathrm{bias}}^2
+
L_{\theta\mu}^2L^2E^2\eta^2
\varepsilon_{\mathrm{ema},T}^2
\right).
\]
Equivalently, if $\tau$ is sampled with
$\Pr(\tau=t)=\gamma_t/\sum_{s=1}^T\gamma_s$, then
\[
\mathbb{E}
\left[
\|\nabla\Phi_\eta(\theta^{(\tau)})\|_2^2
\right]
\le
\widetilde{\mathcal{O}}(T^{-1/2})
+
\mathcal{O}
\left(
\varepsilon_{\mathrm{bias}}^2
+
L_{\theta\mu}^2L^2E^2\eta^2
\varepsilon_{\mathrm{ema},T}^2
\right).
\]
In the idealized population-loss setting, where
$\varepsilon_{\mathrm{bias}}=0$ and
$\varepsilon_{\mathrm{ema},T}\to0$, the expected stationarity measure
vanishes at rate $\widetilde{\mathcal{O}}(T^{-1/2})$.
\end{theorem}

\begin{proof}
By Danskin's theorem applied to the regularized value function,
\[
\nabla\Phi_\eta(\theta^{(t)})
=
\nabla_\theta
F(\theta^{(t)},\mu^\star_\eta(\theta^{(t)})),
\]
because the regularizer in $F_\eta$ does not depend on $\theta$.

Let
\[
\delta_t
=
\mathbb{E}
\left[
g^{(t)}
\mid
\theta^{(t)},\mu^{(t)}
\right]
-
\nabla\Phi_\eta(\theta^{(t)}).
\]
Using Assumption~\ref{as:grad} and the Lipschitz dependence of
$\nabla_\theta F$ on $\mu$,
\[
\|\delta_t\|_2
\le
\varepsilon_{\mathrm{bias}}
+
L_{\theta\mu}
\sum_{l=1}^L
\left\|
\mu_l^{(t)}
-
\mu^\star_{\eta,l}(\theta^{(t)})
\right\|_1 .
\]
Since $\|\cdot\|_1\le E\|\cdot\|_\infty$,
\[
\|\delta_t\|_2
\le
\varepsilon_{\mathrm{bias}}
+
L_{\theta\mu}LE e_t .
\]

Using $L_\Phi$-smoothness of $\Phi_\eta$,
\[
\Phi_\eta(\theta^{(t+1)})
\le
\Phi_\eta(\theta^{(t)})
-
\gamma_t
\left\langle
\nabla\Phi_\eta(\theta^{(t)}),
g^{(t)}
\right\rangle
+
\frac{L_\Phi\gamma_t^2}{2}
\|g^{(t)}\|_2^2 .
\]
Taking conditional expectation and applying
$-\langle a,a+\delta\rangle\le -\frac12\|a\|_2^2+\frac12\|\delta\|_2^2$
gives
\[
\mathbb{E}
\left[
\Phi_\eta(\theta^{(t+1)})
\right]
\le
\mathbb{E}
\left[
\Phi_\eta(\theta^{(t)})
\right]
-
\frac{\gamma_t}{2}
\mathbb{E}
\left[
\|\nabla\Phi_\eta(\theta^{(t)})\|_2^2
\right]
+
\frac{\gamma_t}{2}
\mathbb{E}[\|\delta_t\|_2^2]
+
\frac{L_\Phi\gamma_t^2}{2}M^2 .
\]
Summing over $t=1,\ldots,T$, lower bounding by $\Phi_{\eta,\inf}$, and
using
\[
\|\delta_t\|_2^2
\le
2\varepsilon_{\mathrm{bias}}^2
+
2L_{\theta\mu}^2L^2E^2 e_t^2
\]
gives
\[
\frac{
\sum_{t=1}^T
\gamma_t
\mathbb{E}
\left[
\|\nabla\Phi_\eta(\theta^{(t)})\|_2^2
\right]
}{
\sum_{t=1}^T\gamma_t
}
\le
\frac{
2(\Phi_\eta(\theta^{(1)})-\Phi_{\eta,\inf})
}{
\sum_{t=1}^T\gamma_t
}
+
L_\Phi M^2
\frac{
\sum_{t=1}^T\gamma_t^2
}{
\sum_{t=1}^T\gamma_t
}
+
2\varepsilon_{\mathrm{bias}}^2
+
2L_{\theta\mu}^2L^2E^2
\frac{
\sum_{t=1}^T
\gamma_t\mathbb{E}[e_t^2]
}{
\sum_{t=1}^T
\gamma_t
}.
\]
Substituting the tracking bound from
Lemma~\ref{lem:dual_tracking_appendix} yields
\[
\frac{
\sum_{t=1}^T
\gamma_t
\mathbb{E}
\left[
\|\nabla\Phi_\eta(\theta^{(t)})\|_2^2
\right]
}{
\sum_{t=1}^T\gamma_t
}
\le
\widetilde{\mathcal{O}}(T^{-1/2})
+
\mathcal{O}
\left(
\varepsilon_{\mathrm{bias}}^2
+
L_{\theta\mu}^2L^2E^2\eta^2
\varepsilon_{\mathrm{ema},T}^2
\right),
\]
because for $\gamma_t=\gamma_0/\sqrt{t}$,
\[
\sum_{t=1}^T\gamma_t=\Theta(\sqrt{T}),
\qquad
\sum_{t=1}^T\gamma_t^2=\Theta(\log T),
\qquad
\sum_{t=1}^T\gamma_t^3=\mathcal{O}(1).
\]
This proves the theorem.
\end{proof}

\section{Expert utilization per layer}
We track the Coefficient of Variation (CV) of expert routing weights across training steps for individual layers. The Baseline model exhibits severe transient load instability, characterized by sharp CV spikes in early layers (Layers 2--4) midway through training. In contrast, DRMoET effectively stabilizes the routing dynamics, maintaining a consistent load distribution throughout the entire training trajectory and preventing the early-layer load collapse observed in the baseline.
\begin{figure*}[!htbp]
\centering
  \includegraphics[width=\linewidth]{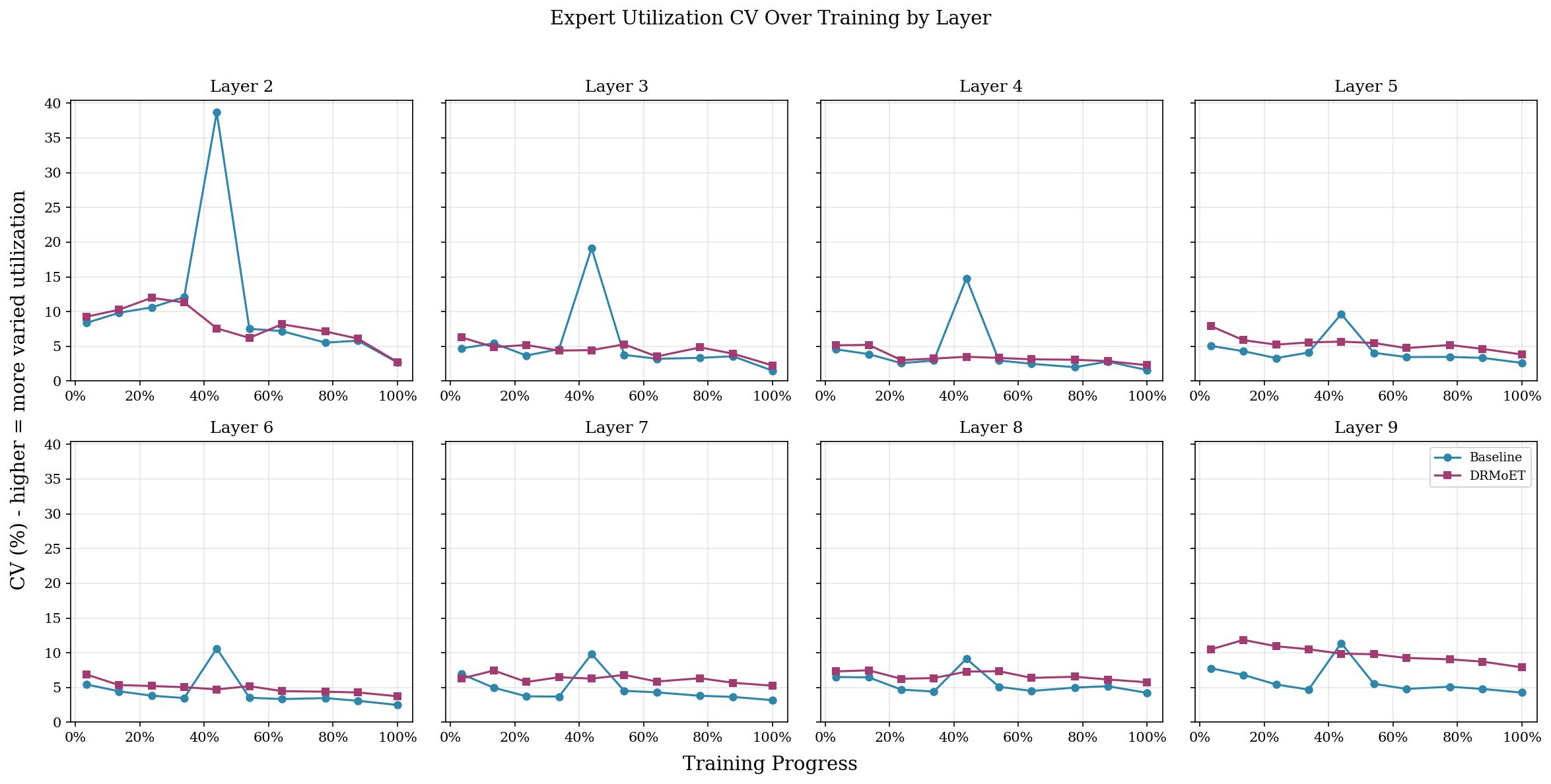}
  \caption{\textbf{Per-layer expert utilization variability over training.} DRMoET stabilizes routing dynamics in early layers and avoids the transient CV spikes observed in the baseline.}
  \label{fig:expert_utilization_per_layer}
\end{figure*}

\FloatBarrier
\section{Expert Utilization CV Over Training}
\label{app:util_cv}

\begin{table}[!htbp]
\centering
\small
\setlength{\tabcolsep}{3pt}
\renewcommand{\arraystretch}{1.0}
\begin{tabular}{rccc}
\toprule
\textbf{Step} & \textbf{Baseline CV} & \textbf{DRMoET CV} & \textbf{Diff} \\
\midrule
  540  & 2.75\% & 3.32\% & +0.57\% \\
 2160  & 2.46\% & 3.04\% & +0.58\% \\
 3780  & 1.78\% & 2.73\% & +0.95\% \\
 5400  & 2.00\% & 2.64\% & +0.64\% \\
 7020  & 6.47\% & 2.21\% & -4.26\% \\
 8640  & 1.60\% & 2.07\% & +0.46\% \\
10260  & 1.32\% & 2.07\% & +0.75\% \\
12420  & 1.43\% & 2.37\% & +0.94\% \\
14040  & 1.44\% & 2.04\% & +0.59\% \\
16000  & 0.91\% & 1.58\% & +0.67\% \\
\midrule
\textbf{Average} & \textbf{2.22\%} & \textbf{2.41\%} & \\
\bottomrule
\end{tabular}
\caption{\textbf{Expert utilization CV over training} (six routed experts plus two shared experts). Higher CV indicates greater relative variation in load across experts. DRMoET shows slightly higher average CV but avoids the baseline's transient mid-training imbalance spike (step 7k), consistent with controlled exploration early and modest specialization later.}
\label{tab:util_cv}
\end{table}

\FloatBarrier
\section{Coactivation Concentration Metrics}
\label{app:coactivation}

\begin{figure*}[!htbp]
  \centering
  \includegraphics[width=\linewidth]{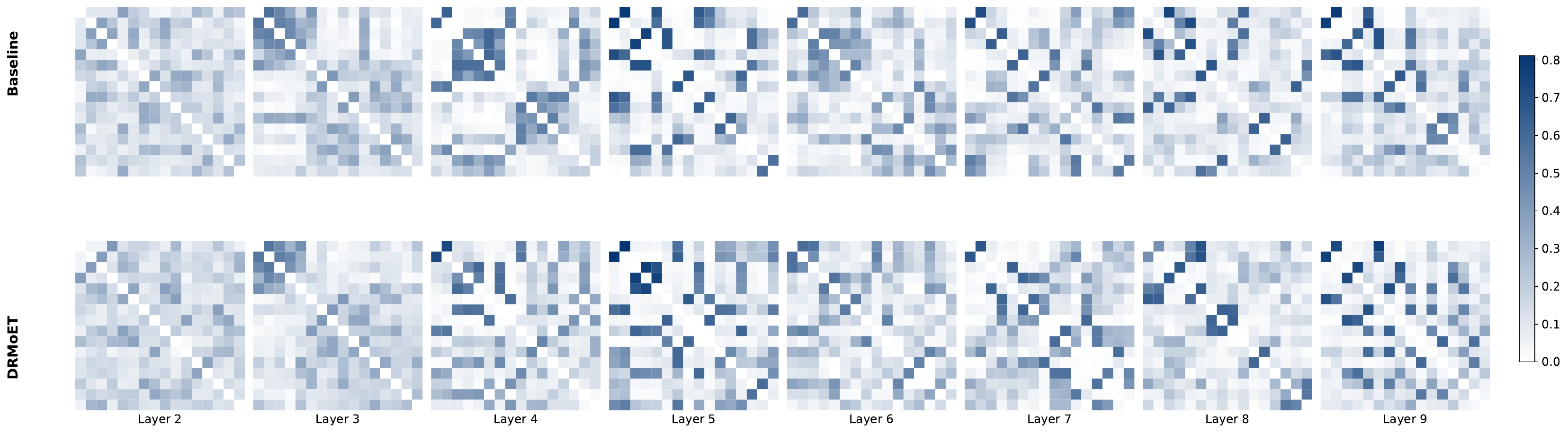}
  \caption{\textbf{Expert co-activation heatmaps} (layer-wise) comparing baseline vs.\ DRMoET. DRMoET produces more uniform co-activation in early/mid layers (L2--L5) and slightly more concentrated patterns in later layers (L6--L9).}
  \label{fig:coactivation}
\end{figure*}

\begin{table*}[!htbp]
\centering
\scriptsize
\setlength{\tabcolsep}{4pt}
\begin{tabular}{c|ccc|ccc|ccc}
\toprule
\multirow{2}{*}{Layer} &
\multicolumn{3}{c|}{Variance} &
\multicolumn{3}{c|}{Gini} &
\multicolumn{3}{c}{Top-3 Conc.} \\
& Baseline & DRMoET & $\Delta$ & Baseline & DRMoET & $\Delta$ & Baseline & DRMoET & $\Delta$ \\
\midrule
2 & 0.00518 & 0.00483 & -6.6\%  & 0.248 & 0.238 & -4.2\%  & 17.7\% & 17.5\% & -1.2\% \\
3 & 0.00800 & 0.00678 & -15.3\% & 0.301 & 0.275 & -8.4\%  & 19.7\% & 19.3\% & -2.4\% \\
4 & 0.01534 & 0.01395 & -9.0\%  & 0.413 & 0.393 & -4.9\%  & 24.3\% & 23.8\% & -2.4\% \\
5 & 0.02487 & 0.02188 & -12.0\% & 0.500 & 0.475 & -5.0\%  & 30.7\% & 28.7\% & -6.4\% \\
6 & 0.01477 & 0.01549 & +4.9\%  & 0.422 & 0.427 & +1.1\%  & 24.0\% & 25.1\% & +4.3\% \\
7 & 0.01694 & 0.01719 & +1.5\%  & 0.430 & 0.438 & +1.9\%  & 26.3\% & 26.0\% & -1.2\% \\
8 & 0.01886 & 0.01862 & -1.2\%  & 0.443 & 0.451 & +1.6\%  & 28.4\% & 27.7\% & -2.4\% \\
9 & 0.01902 & 0.02105 & +10.7\% & 0.451 & 0.465 & +3.2\%  & 27.7\% & 29.4\% & +5.9\% \\
\bottomrule
\end{tabular}
\caption{\textbf{Coactivation concentration metrics (DRMoET vs.\ Baseline).} Negative $\Delta$ = more uniform partnerships under DRMoET; positive $\Delta$ = increased specialization.}
\label{tab:coactivation_metrics}
\end{table*}

\FloatBarrier
\section{Interaction with the standard load-balancing loss}
\label{app:load-balancing-ablation}

At the 290M-active/746M-total scale, we remove the standard load-balancing loss from DRMoET while keeping the remaining reported training configuration unchanged. Table~\ref{tab:load_balancing_ablation} uses the same seven-task evaluation as the main results. The no-balancing DRMoET configuration reaches 0.5448, exceeding standard FLAME-MoE's 0.5403 on the average and on five tasks. DRMoET with balancing remains stronger on the average at 0.5523.

\begin{table}[ht]
\centering
\scriptsize
\setlength{\tabcolsep}{3pt}
\resizebox{\textwidth}{!}{%
\begin{tabular}{lccccccc|c}
\toprule
\textbf{Method} & \textbf{ARC-C} & \textbf{ARC-E} & \textbf{HellaSwag} & \textbf{PIQA} & \textbf{WinoGrande} & \textbf{SciQ} & \textbf{ReCoRD*} & \textbf{Average} \\
\midrule
FLAME-MoE (with balancing) & 0.2329 & 0.5295 & 0.3447 & 0.6697 & 0.5012 & 0.8020 & 0.7023 & 0.5403 \\
DRMoET without balancing & \textbf{0.2389} & 0.5539 & 0.3430 & 0.6877 & \textbf{0.5067} & 0.7770 & 0.7066 & 0.5448 \\
DRMoET with balancing & 0.2295 & \textbf{0.5682} & \textbf{0.3452} & \textbf{0.6888} & 0.5051 & \textbf{0.8190} & \textbf{0.7105} & \textbf{0.5523} \\
\bottomrule
\end{tabular}%
}
\caption{\textbf{Load-balancing interaction at 290M-active/746M-total scale.} Averages are over seven tasks. Standard FLAME-MoE includes balancing; a FLAME-MoE-without-balancing control is not reported. *ReCoRD is F1.}
\label{tab:load_balancing_ablation}
\end{table}

The robust objective is defined on expert outcomes that occur under the router. It does not require a separate balancing penalty to be evaluated, but cannot by itself ensure adequate observations for experts that receive almost no traffic. Allocation and observed competence therefore remain distinct concerns. The results support using the two objectives together in this recipe; they do not establish that DRMoET replaces balancing or quantify its incremental effect relative to an unreported no-balancing baseline.

\FloatBarrier
\section{Implementation Details}
\label{appendix:implementation-details}

\textbf{Training-token budgets.}
The \texttt{FLAME-MoE-290M-746M} experiments use approximately 33.5B training tokens. The \texttt{FLAME-MoE-1.7B-10.3B} comparison uses approximately 67B tokens. Within each scale and budget, the compared methods use the same architecture, data mixture, and training-token count.

We used $1\times10^{-3}$ for both auxiliary loss, and z-loss as training hyperparameters that all runs share.

\textbf{Data.}
The \texttt{FLAME-MoE-290M-746M} experiments use the same data the original \citet{kang2025flame} paper has used. For the \texttt{FLAME-MoE-1.7B-10.3B} run, due to the increment on \#tokens trained, we downloaded and processed more chunks of DCLM to guarantee we stick to the 1-epoch only training setup, which is commonly used while training LLMs.

\newpage
\section*{NeurIPS Paper Checklist}

\begin{enumerate}

\item {\bf Claims}
    \item[] Question: Do the main claims made in the abstract and introduction accurately reflect the paper's contributions and scope?
    \item[] Answer: \answerYes{} 
    \item[] Justification: The abstract and introduction explained how does DRMoET work, and to what scope it answers the problem the paper is trying to address.

\item {\bf Limitations}
    \item[] Question: Does the paper discuss the limitations of the work performed by the authors?
    \item[] Answer: \answerYes{} 
    \item[] Justification: The paper mentioned some of the limitations, such as lack of larger scale models, different expert counts, and also train time analysis on OOD data.

\item {\bf Theory assumptions and proofs}
    \item[] Question: For each theoretical result, does the paper provide the full set of assumptions and a complete (and correct) proof?
    \item[] Answer: \answerYes{} 
    \item[] Justification: For the proof of convergence, the whole proof is included in the appendix. The assumptions, and lemma, were provided in detail.

    \item {\bf Experimental result reproducibility}
    \item[] Question: Does the paper fully disclose all the information needed to reproduce the main experimental results of the paper to the extent that it affects the main claims and/or conclusions of the paper (regardless of whether the code and data are provided or not)?
    \item[] Answer: \answerYes{} 
    \item[] Justification: Training configurations such as the pipeline used, what is the configuration of the model, what are the hyper-parameter used, etc. The detailed pseudocode was provided to ensure reproducibility.

\item {\bf Open access to data and code}
    \item[] Question: Does the paper provide open access to the data and code, with sufficient instructions to faithfully reproduce the main experimental results, as described in supplemental material?
    \item[] Answer: \answerYes{} 
    \item[] Justification: The code is provided as supplementary material. The entire pipeline is included.

\item {\bf Experimental setting/details}
    \item[] Question: Does the paper specify all the training and test details (e.g., data splits, hyperparameters, how they were chosen, type of optimizer) necessary to understand the results?
    \item[] Answer: \answerYes{} 
    \item[] Justification: The dataset used, and the reason we chose it is explained. We also provided the hyperparameters we used, as well as the optimizer chosen.

\item {\bf Experiment statistical significance}
    \item[] Question: Does the paper report error bars suitably and correctly defined or other appropriate information about the statistical significance of the experiments?
    \item[] Answer: \answerYes{} 
    \item[] Justification: To ensure the significance of our method, we tested it across several hyperparameter settings, and in all cases, it outperforms the baseline. This show that our method's gain is consistent.

\item {\bf Experiments compute resources}
    \item[] Question: For each experiment, does the paper provide sufficient information on the computer resources (type of compute workers, memory, time of execution) needed to reproduce the experiments?
    \item[] Answer: \answerYes{} 
    \item[] Justification: We did mention the GPU we use, as well as how many steps were trained. The two together can provide a full picture of the cost for training one model in our setting.
    
\item {\bf Code of ethics}
    \item[] Question: Does the research conducted in the paper conform, in every respect, with the NeurIPS Code of Ethics \url{https://neurips.cc/public/EthicsGuidelines}?
    \item[] Answer: \answerYes{} 
    \item[] Justification: All experiments done are on the GPU, without causing social harm.

\item {\bf Broader impacts}
    \item[] Question: Does the paper discuss both potential positive societal impacts and negative societal impacts of the work performed?
    \item[] Answer: \answerNo{} 
    \item[] Justification: The paper introduces a pretraining algorithm, which is irrelevant to broader societal impact.
    
\item {\bf Safeguards}
    \item[] Question: Does the paper describe safeguards that have been put in place for responsible release of data or models that have a high risk for misuse (e.g., pre-trained language models, image generators, or scraped datasets)?
    \item[] Answer: \answerNA{} 
    \item[] Justification: The paper does not provide data/model that have a high risk for misuse. The training data is from DCLM, which is a public-available language dataset, that is safe to use/release.

\item {\bf Licenses for existing assets}
    \item[] Question: Are the creators or original owners of assets (e.g., code, data, models), used in the paper, properly credited and are the license and terms of use explicitly mentioned and properly respected?
    \item[] Answer: \answerYes{} 
    \item[] Justification: All previous related work is cited in the paper. We also mentioned explicitly what code base did we build our code base on top of.

\item {\bf New assets}
    \item[] Question: Are new assets introduced in the paper well documented and is the documentation provided alongside the assets?
    \item[] Answer: \answerNA{} 
    \item[] Justification: No new assets was release from the paper.

\item {\bf Crowdsourcing and research with human subjects}
    \item[] Question: For crowdsourcing experiments and research with human subjects, does the paper include the full text of instructions given to participants and screenshots, if applicable, as well as details about compensation (if any)? 
    \item[] Answer: \answerNA{} 
    \item[] Justification: The paper does not involve crowdsourcing nor research with human subjects.

\item {\bf Institutional review board (IRB) approvals or equivalent for research with human subjects}
    \item[] Question: Does the paper describe potential risks incurred by study participants, whether such risks were disclosed to the subjects, and whether Institutional Review Board (IRB) approvals (or an equivalent approval/review based on the requirements of your country or institution) were obtained?
    \item[] Answer: \answerNA{} 
    \item[] Justification: the paper does not involve crowdsourcing nor research with human subjects.

\item {\bf Declaration of LLM usage}
    \item[] Question: Does the paper describe the usage of LLMs if it is an important, original, or non-standard component of the core methods in this research? Note that if the LLM is used only for writing, editing, or formatting purposes and does \emph{not} impact the core methodology, scientific rigor, or originality of the research, declaration is not required.
    \item[] Answer: \answerNA{} 
    \item[] Justification: the core method development in this research does not involve LLMs as any important, original, or non-standard components.

\end{enumerate}

\end{document}